\documentclass{article}

\PassOptionsToPackage{numbers, sort&compress}{natbib}
\usepackage[preprint]{neurips_2023}

\usepackage[utf8]{inputenc} 
\usepackage[T1]{fontenc}    
\usepackage{hyperref}       
\usepackage{url}            
\usepackage{booktabs}       
\usepackage{amsfonts}       
\usepackage{nicefrac}       
\usepackage{microtype}      
\usepackage{xcolor}         
\usepackage{graphicx} 

\usepackage{amsmath}
\usepackage{algorithm}
\usepackage{algorithmicx}
\usepackage{algpseudocode}
\usepackage{subfigure}

\usepackage{enumitem}
\setlist{leftmargin=5.5mm}

\usepackage{color}
\usepackage{colortbl}
\definecolor{highlight}{gray}{0.9}

\usepackage{multicol}
\usepackage{array,multirow}
\usepackage{lipsum}
\usepackage{wrapfig,lipsum,booktabs}

\usepackage{xspace}

\usepackage{amsthm}
\usepackage[english]{babel}
\newtheorem{theorem}{Theorem}

\newtheorem{assumption}{Assumption}

\newtheorem{definition}{Definition}

\newcommand\normx[1]{\Vert#1\Vert}
\newcommand\ca{c_1}
\newcommand\cb{c_2}

\newcommand\fedavg{\texttt{FedAvg}\xspace}
\newcommand\fedbug{\texttt{FedBug}\xspace}

\title{FedBug: A Bottom-Up Gradual Unfreezing Framework for Federated Learning}

\author{%
    Chia-Hsiang Kao$^1$
    \\
  $^1$Cornell University\\
  \texttt{ck696@cornell.edu} \\
  \And
  Yu-Chiang Frank Wang$^{2}$  \\
  $^2$National Taiwan University \\
  \texttt{ycwang@ntu.edu.tw} \\
}

\begin{document}

\maketitle

\begin{abstract}
Federated Learning (FL) offers a collaborative training framework, allowing multiple clients to contribute to a shared model without compromising data privacy. 
Due to the heterogeneous nature of local datasets, updated client models may overfit and diverge from one another, commonly known as the problem of client drift.
In this paper, we propose FedBug (Federated Learning with Bottom-Up Gradual Unfreezing), a novel FL framework designed to effectively mitigate client drift.
FedBug adaptively leverages the client model parameters, distributed by the server at each global round, as the reference points for cross-client alignment. 
Specifically, on the client side, FedBug begins by freezing the entire model, then gradually unfreezes the layers, from the input layer to the output layer.
This bottom-up approach allows models to train the newly thawed layers to project data into a latent space, wherein the separating hyperplanes remain consistent across all clients. 
We theoretically analyze FedBug in a novel over-parameterization FL setup, revealing its superior convergence rate compared to FedAvg.
Through comprehensive experiments, spanning various datasets, training conditions, and network architectures, we validate the efficacy of FedBug. 
Our contributions encompass a novel FL framework, theoretical analysis, and empirical validation, demonstrating the wide potential and applicability of FedBug.\footnote{Code is available at \href{https://github.com/IandRover/FedBug}{https://github.com/IandRover/FedBug}}.
\end{abstract}

\section{Introduction}

Federated Learning (FL) is a distributed approach that enables multiple clients to collaboratively train a shared model without disclosing their raw data. Federated Average (\fedavg)~\cite{mcmahan2017communication}, one of the most influential FL frameworks, has served as the cornerstone for numerous algorithms in the field. Below, we provide a concise explanation of how \fedavg operates: It involves a central server and several clients. In each global round, the server disseminates the current model to all clients. Each client independently trains its model using its local data until convergence. Once the local training is completed, the clients transmit their models back to the server. The server then averages these models to produce an updated global model, which is subsequently employed in the next round.

In FL, \textit{client drift} refers to the inconsistency between models learned by different clients, arising primarily due to the disparities in their private data distribution~\cite{karimireddy2020scaffold, luo2021no, li2022partial, guo2022fedaug}. As local models overfit to their datasets and converge to local minima, the global model --- derived from averaging client models --- compromises in terms of convergence and performance~\cite{li2020federated, zhao2018federated, zhang2022federated}. 
Extensive studies have developed strategies to tackle the client drift issue. We specifically focus on those utilizing \textit{anchors} shared among clients, including gradient anchors and feature anchors. 
Gradient anchors involve the use of shared gradient information to guide the update of the client model, thereby promoting alignment and mitigating client drift~\cite{karimireddy2020scaffold, xu2021fedcm, das2022faster, karimireddy2020mime, li2019feddane}. On the other hand, feature anchors rely on shared feature information to assist in feature alignment and regularization of the feature space~\cite{luo2021no, tang2022virtual, tan2022fedproto, Xu2023personalized}.
However, these methods may necessitate the extra transmission of gradient information or increased regularization costs~\cite{karimireddy2020mime, xu2021fedcm, karimireddy2020scaffold, li2019feddane} and pose privacy concerns~\cite{luo2021no, tan2022fedproto, Xu2023personalized}. 

{
To mitigate the client drift problem, another line of research has focused on leveraging model parameter anchors with enhanced training efficiency. For example, FedBABU~\cite{oh2021fedbabu} proposes to fix the classifier throughout training and update it only during evaluation. As all clients share the same fixed classifier, a set of decision boundaries is common to all clients, serving as a shared reference for updating the encoder. 
While FedBABU yields promising results in personalized FL scenarios, in which models are allowed to be fine-tuned using the client's private data during the evaluation stage, FedBABU's performance in general FL settings is less optimal and lacks theoretical understanding.
}

{
In this work, we seek to extend FedBABU by including more intermediate layers as reference points, ensuring trainability in general FL scenarios and providing theoretical support. Our approach hinges on two insights: (1) At the start of each global round of \fedavg, all clients receive an identical model from the server, and (2) each intermediate layer parameterizes hyperplanes that separate latent features. 
Taken together, these insights suggest a strategy: By freezing the models received from the server, we can exploit the consistency of the hyperplanes across clients to provide a common feature space for alignment. 
}

{
Building on these insights, we introduce \textbf{\fedbug} (\underline{\textbf{Fed}}erated Learning with \underline{\textbf{B}}ottom-\underline{\textbf{U}}p \underline{\textbf{G}}radual Unfreezing), a novel FL framework leveraging shared parameter anchors to mitigate client drift. 
Unlike \fedavg, \fedbug begins local training by freezing the entire model, then gradually thaws the layers from the input layer to the output layer. 
The key mechanism operates as follows: when a layer becomes trainable (thawed) while its succeeding layers remain frozen, this thawed layer learns to project its inputs into a shared feature space. This space is notably defined by the hyperplanes of the still-frozen succeeding layers, providing a common reference across clients and enhancing cross-client alignment. Then, \fedbug progressively unfreezes the next layer, ensuring the layer's trainability after it has served as a shared reference. This results in a framework that balances alignment and adaptability. 
}

{
We conduct a thorough investigation of \fedbug through theoretical analysis and extensive experiments.
Notably, we are the first to analyze FL algorithms in an over-parameterized setting, pushing the boundaries of the current FL theoretical framework. Utilizing a two-layer linear network with an orthogonal regression task setting, our analysis reveals that \fedbug converges faster than \fedavg. Additionally, we conduct extensive experiments across various datasets (CIFAR-10, CIFAR-100, Tiny-ImageNet), training conditions (label skewness and client participation rate), and network architectures (simple CNN, Resnet-18, Resnet34) to ensure the broad applicability of \fedbug.
}

Our contributions are three-fold:

\begin{itemize}
\item \textbf{Novel Federated Learning Framework:} We propose \fedbug, a novel federated learning framework that leverages model parameters as anchors to effectively address the challenges of client drift in federated learning.
\item \textbf{Theoretical Analysis:} We provide a theoretical analysis of the \fedbug framework, demonstrating its better convergence rate compared to \fedavg. Our analysis focuses on a two-layer linear network with an orthogonal regression task, offering novel insights into federated learning dynamics in the context of over-parameterized models.
\item \textbf{Empirical Validation:} We extensively validate the effectiveness of \fedbug through a series of experiments on diverse datasets (CIFAR-10, CIFAR-100, Tiny-ImageNet), varying training conditions (label skewness, client participation rate), and different model architectures (standard CNN, ResNet18, ResNet34). Furthermore, we assess the compatibility of the \fedbug framework with other federated learning algorithms. Our empirical findings underscore the immense potential and wide-ranging applicability of \fedbug.
\end{itemize}
\section{Literature Review}
\textbf{Mitigating Client Drift Using Gradient and Feature Anchors.}
To address the client drift problem in federated learning, various methods have been explored to explicitly or implicitly provide shared reference points for client models. \textbf{Gradient anchors} are employed in several FL methods. For instance, SCAFFOLD~\cite{karimireddy2020scaffold} incorporates server-level gradients in updates to reduce local noise effects, while FedCM~\cite{xu2021fedcm}, FedGLOMO~\cite{das2022faster}, MIME~\cite{karimireddy2020mime}, and FedDANE~\cite{li2019feddane} leverage server-level gradient to align clients by providing mutual update directions. 
While these approaches have demonstrated effectiveness, they require sharing additional gradient information with the local clients. \textbf{Feature anchors} are used in methods such as CCVR~\cite{luo2021no}, VHL~\cite{tang2022virtual}, FedProto~\cite{tan2022fedproto}, and FedFA~\cite{Xu2023personalized}. These methods employ outside datasets or clients' private datasets to produce and regularize features across different clients. However, these approaches may suffer from privacy leakage, increased dataset or feature transmission costs, and added computation budget. 

\textbf{Mitigating Client Drift Using Parameter Anchor.}
Recent research has emphasized the prominence of client drift in the top layers of models. 
Specifically,~\cite{zhao2018federated, luo2021no, li2022partial, guo2022fedaug} demonstrate that the penultimate layer and classifier exhibit the lowest feature similarities among the clients. These findings suggest that local classifiers undergo significant changes to adapt to the local data distribution, which amplifies challenges related to class imbalance and results in biased model predictions for specific classes~\cite{zhang2022federated, shang2022federated, lee2021preservation, guo2022fedaug}. 
This phenomenon is consistent with research on the long-tail analysis, where~\cite{yu2020devil, kang2020decoupling} demonstrated that the head is biased in class-imbalanced environments. 
To explicitly address the issue of client drift, certain studies leverage parameter anchors to align the clients. 
FedProx~\cite{li2020federated} regularizes the L2 distance between the client model and server model, establishing a shared reference in the parameter space. 
Additionally, FedBABU~\cite{oh2021fedbabu} falls into this category as it uses a fixed classifier as the parameter-level anchor during training. 
As each client's frozen classifier parameterizes the same set of decision boundaries in the feature space, FedBABU also serves as a hyperplane anchor in the latent space. However, FedBABU's performance is suboptimal and necessitates additional personalized training (fine-tuning) on the clients' private dataset during the evaluation stage, a process known as personalization, to achieve satisfactory results.

\section{Method}

\begin{figure}[t]
\centering
\includegraphics[width=.95\textwidth]{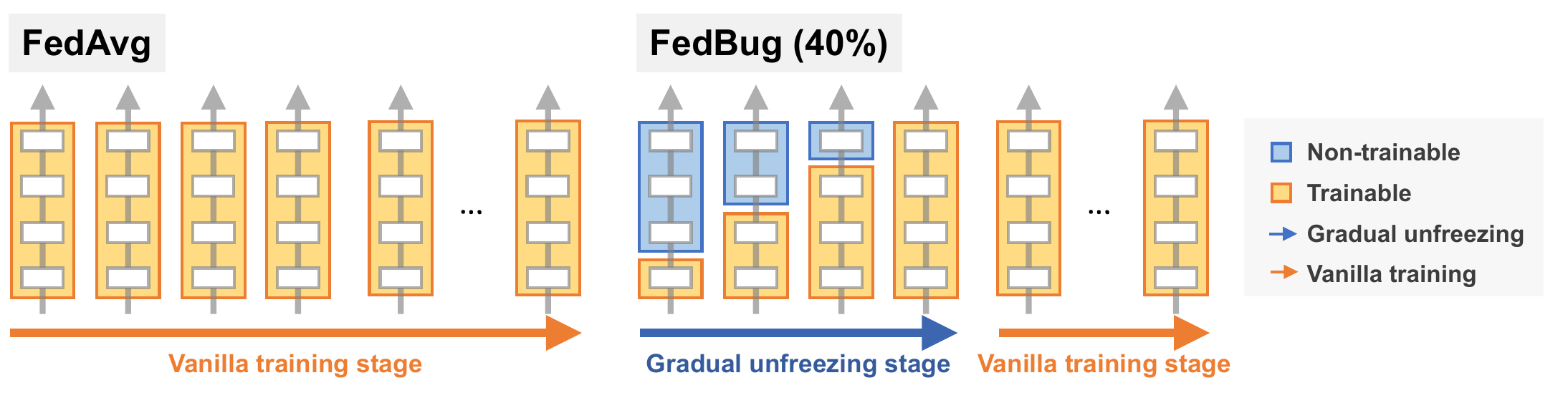}
\vspace{-2mm}
\caption{
    \textbf{Comparisons of \fedavg and \fedbug.}
    While \fedavg updates all network modules during local training, \fedbug gradually unfreezes network modules in a bottom-up manner to tackle client drift. Take \fedbug (40\%) for example, the first 40\% of training iterations perform gradual unfreezing (GU), while the remaining 60\% perform vanilla training. With the same number of training iterations, \fedbug has fewer parameters to update and thus exhibits improved learning efficiency.
}
\label{fig:alg_fedbug}
\end{figure}

\fedbug is designed to pivot parameter anchors to reduce client drift. 
As shown in Figure~\ref{fig:alg_fedbug}, \fedbug divides the local training phase into the gradual unfreezing (GU) stage and the vanilla training stage. In the vanilla training stage, all modules of a model are trainable. 
During the GU stage, \fedbug freezes the whole model initially and then progressively unfreezes its modules bottom-up. This stage is equally divided into several periods, with an additional module being unfrozen in each period.

We first introduce the notations used in this work. 
Consider $N$ clients, where each client $i$ has its own loss function $F_i$. 
The model is denoted as $\theta$ and is composed of $m$ modules, where a \textit{module} can refer to a single CNN layer or a ResNet block.
For simplicity, we use $\theta^{1:m}$ to signify the first $m$ modules of the model, with $\theta^1$ indicating the input module. 
We define $P$ as the percentage of the GU stage and $K$ as the total number of local training iterations, hence the GU stage spans $PK$ local iterations. We denote $\eta_g$ and $\eta_l$ as global and local learning rates, respectively.
Algorithm~\ref{alg:fedbug} provides a detailed description of the \fedbug framework. In line seven of the algorithm, the variable $m$ is determined based on the current local iteration step $k$ and specifies the number of modules to be updated. For a comprehensive comparison between \fedavg and \fedbug, please refer to Appendix~\ref{sec:app_alg}.

\begin{multicols}{2}

We now explore the underlying rationale of \fedbug, using Figure~\ref{fig:alg_fedbug} for illustration. Consider a four-module model trained using \fedbug (40\%), represented as $M=4$, and $P=0.4$. In this case, a module becomes trainable every $0.1K$ local iteration. Suppose we are in the second GU period, where all clients have just unfrozen their second module. During this period, the clients adapt their first and second modules and project the data into a feature space. Notably, the separating hyperplanes within this feature space are parameterized by the yet-to-be-unfrozen modules (the third and fourth modules in this case). These modules remain consistent during this period, serving as a shared anchor among clients. 
Similarly, as we progress to the subsequent third period, this process continues, with clients mapping their data into decision regions defined by the still-frozen fourth module. By leveraging the shared reference, \fedbug ensures ongoing alignment among the clients.

\columnbreak
\begin{algorithm}[H]
    \caption{\texttt{FedBug: Federated Learning with Bottom-Up Gradual Unfreezing}}
    \begin{algorithmic}[1]
        \Statex \textbf{Notation}: 
        \Statex \hspace{10px} $\theta^{1:m}$: the first $m$ modules of model $\theta$ 
        \Statex \hspace{10px} $R$: number of global rounds
        \Statex \hspace{10px} $K$: number of local iterations
        \Statex \hspace{10px} $P$: gradual unfreezing stage percentage
        \State \textbf{Input}: global model $\theta$ with $M$ modules
        \For {$r = 1, \dots, R$}
        \State Sample clients $S \subseteq \{1,...,N\}$
        \For {each client $i \in S$ in parallel}
        \State Initialize local model $\theta_{i} \gets \theta$
        \For {$k=1,\dots, K$}
        \State $m \gets \min \{M, \lceil \frac{kM}{PK} \rceil \}$
        \State //~Update $m$ modules of $\theta_{i}$
        \State $\theta_{i}^{1:m}  \gets \theta_{i}^{1:m} - \eta_l\nabla F_i (\theta_{i}^{1:m})$ 
        \EndFor
        \State $\Delta_i \gets \theta_{i} - \theta$
        \EndFor
        \State $\theta \gets \theta + \frac{\eta_g }{\lvert S \rvert}\sum_{i \in \mathcal{S}}\Delta_i$
        \EndFor
    \end{algorithmic}
    \label{alg:fedbug}
\end{algorithm}
\end{multicols}

\section{Theoretical Analysis}
In this section, we present a theoretical analysis of the convergence of \fedbug. We start with a brief overview, explaining the motivation behind our theoretical setup. We then discuss the technical challenges and elaborate on the key intuitions guiding our approach. Finally, we present the results of our analysis, which offer valuable insights into the convergence behavior of \fedbug within the theoretical framework we have considered.

\textbf{Motivation 1: Extending From Single-Variable Loss Function.}
As our \fedbug leverages a shared feature space across clients during local iterations, an theoretical understanding of it requires at least two neural layers so that the concept of feature space is valid. 
However, conventional FL theoretical frameworks often restrict to single-variable loss functions for clients, which may not adequately capture the dynamics of \fedbug. To bridge this gap, we introduce an over-parameterized analytic FL framework that provides a more comprehensive understanding of the FL dynamics.

\textbf{Motivation 2: Addressing Client Drift via Dataset Heterogeneity.} 
While previous FL analyses mainly tackle label distribution~\cite{li2019feddane, morafah2022rethinking}, clients' loss functions~\cite{acar2021federated, karimireddy2020scaffold, xu2021fedcm}, and local training conditions~\cite{wang2020tackling} as sources of client drift, limited attention has been given to client drift caused by dataset distributions.
Thus, we aim to explore the impact of dataset domain heterogeneity on client drift. This heterogeneity becomes apparent in over-parameterized setups, where the models have sufficient capacity to capture different aspects of the data distributions across clients. 

\textbf{Theoretical Setup and Technical Challenges.} 
Our analysis centers on regression tasks with linear neural network~\citep{saxe2013exact, gunasekar2017implicit, arora2018optimization, gidel2019implicit, tian2021understanding, jing2022understanding, ye2023freeze}, focusing specifically in an orthogonal dataset distribution. This particular setup has been studied in the context of transfer learning with out-of-distribution dataset~\cite{Kumar2022fine} and domain generalization~\cite{lee2023surgical}. For more detailed information about the task setup, please refer to Appendix~\ref{subsec:app_thm_setup}.
Analyzing FL algorithms in the context of a two-layer linear network with an orthogonal dataset distribution presents several challenges. As we show in Appendices~\ref{subsec:app_thm_dynamics} and~\ref{subsec:app_thm_challenges}, over-parameterization results in non-convexity~\cite{safran2021effects} and hyperbolic trajectories of local parameters update, making the analysis difficult. Moreover, as we focus on the FL setup, the convergence of server model parameters must be considered, further complicating the problem.

\textbf{Intuition and Assumptions.} 
To tackle these challenges, we follow~\cite{lee2023surgical} and consider a simplified scenario with two clients. Additionally, we propose three assumptions: (1) linear approximation of the global minimum, (2) bounding of local convexity, and (3) linear approximation of local learning trajectories. 
These assumptions are akin to those utilized in previous FL theoretical frameworks, such as bounded gradient norm, convexity, smoothness, cross-client gradient, or Hessian dissimilarity~\cite{li2019feddane,li2020federated,karimireddy2020mime,xu2021fedcm,reddi2020adaptive,tong2020effective}. 
They serve to streamline the analysis process while capturing essential aspects of convergence speed characterization. 
Although the proposed assumptions may not fully capture the complexity of the problem, they provide a valuable starting point for understanding the dynamics of FL in over-parameterized models.

\textbf{Theoretical Results.} Our analysis demonstrates that, while \fedavg converges linearly to a global minimizer, its convergence is limited by a mismatch between the global and client update directions. In contrast, our \fedbug introduces a single parameter freezing step that encourages greater alignment between the local and global update directions, leading to significantly improved convergence. It offers a theoretical foundation for the effectiveness of the \fedbug algorithm and sheds light on the dynamics of FL in this challenging setup.

In the following subsections, we detail the task and model setup for our analysis in Section~\ref{subsec:thm_setup}, present the key theoretical analysis in Section~\ref{subsec:thm_result}, and describe the simulation experiments in Section~\ref{subsec:thm_exp}. For a comprehensive theoretical analysis, please refer to Appendix~\ref{sec:app_thm}.

\subsection{Task Setting and Model Architecture}
\label{subsec:thm_setup}


To ensure tractability in characterizing \fedbug, we focus on a two-layer linear network as the FL model $f$, with an orthogonal dataset distribution as a source of client drifts. We consider a FL regression task, with two clients denoted as $c_1$ and $c_2$. Each client has different regression data, specifically $\mathcal{T}_1=\{x_1=[1,0], y_1=1\}$ and $\mathcal{T}_2=\{x_2=[0,1], y_2=1\}$. The objective is to minimize the L2 loss, with client $c_1$ ($c_2$) minimizing $L_1=\normx{f(x_1)-y_1}$ ($L_2=\normx{f(x_2)-y_2}^2$). Note that the network $f$ is with two nodes $[a, b]$ in the first layer and one node $[v]$ in the second layer. Specifically, $f(x) = x[a,b]^\top v$. The task setup implies that client $\ca$ ($\cb$) aims to minimize $L_1=|av-1|^2$ ($L_2=|bv-1|^2$).

\subsection{Convergence Rates of \fedavg and \fedbug}
\label{subsec:thm_result}

In order to analyze the convergence properties of both algorithms, we focus on the region around the minima and make additional assumptions for tractability. These assumptions characterize the landscape near the server model, local client models, and global minima.

\begin{assumption} \label{asm_LocalLinear}
\textbf{(Local Linearity)} The server model parameters are initialized in the vicinity of the global minimum, such that the global minimum can be locally approximated as a linear function.
\end{assumption}

\begin{assumption} \label{asm_BoundC}
\textbf{(Bounded Local Convexity)} Under Assumption~\ref{asm_LocalLinear}, there exist constants $\beta_1 > 0$ and $0 < \beta_2 < 1$, such that for all $n$, the value of ${{v^n}^2}$ is bounded as follows: $\beta_1\beta_2 \leq {{v^n}^2} \leq \beta_1$.
\end{assumption}

\begin{assumption}
\label{asm_OrthoTrajectory}
\textbf{(Orthogonal Trajectory and Bounded Error)} Under Assumptions~\ref{asm_LocalLinear} and~\ref{asm_BoundC}, with a properly chosen step size, the local gradient descent update trajectory can be approximated by a linear trajectory orthogonal to the global minimum. The approximation error, which quantifies the discrepancy between the gradient descent solution and the ideal orthogonal linear path solution, is bounded by $\alpha$ times the length of the orthogonal linear trajectory.
\end{assumption}

Assumption~\ref{asm_LocalLinear} ensures that \fedavg and \fedbug starts close enough to the global minimum, enabling convergence without the interference of poor initialization and unbounded convexity. 
Building on Assumption~\ref{asm_LocalLinear}, Assumption~\ref{asm_BoundC} further bound the local convexity around the global minimum, so that both \fedavg and \fedbug are prevented from diverging or oscillating around the global minimum when using an appropriate step size.
Lastly, Assumption~\ref{asm_OrthoTrajectory} guarantees that \fedbug converges along a favorable direction and helps us to quantify the improvement brought by a single global round while allowing for quantification of the approximation error. Please refer to the right panel of Figure~\ref{fig:thm_illustration} for illustration. We defer the discussion of limitation to Section~\ref{sec:thm_disscussion}.

Distinct from conventional FL convergence analysis, we discover that our unique setup permits an alternative approach for proving the convergence behavior of \fedavg. The cornerstone of our theorem is the observation that at each global round, the discrepancy between clients' parameters diminishes by a factor of $r$. Here, $r$ is determined by the degree of alignment between the local update and the axial direction of the last layer parameters –-- a measure that captures the consistency between the local and the global updates.

To measure this discrepancy reduction ratio, we define two useful terms:

\begin{definition}
\label{def_one}
Client discrepancy $d^i$ is the L1 distance between the server model parameters $a^i$ and $b^i$ at the $i$-th global round : $d^i = \normx{a^i-b^i}$.
\end{definition}

\begin{definition}
\label{def_two}
Client discrepancy contraction ratio $r$: $r=d^{i+1}/d^{i}$.
\end{definition}

Now, we present the theorem describing the convergence behavior of \fedavg:
\begin{theorem} \label{thm_Fedavg}
Under Assumptions~\ref{asm_LocalLinear},~\ref{asm_BoundC} and~\ref{asm_OrthoTrajectory}, models trained using \fedavg converge with the client discrepancy contraction ratio upper bounded by $\frac{1+\cos^2\theta (1+\alpha)}{2}$, where $\theta$ is the angle between the local update direction and the axis of the last layer parameter.
\end{theorem}

\begin{theorem} \label{thm_Fedbug}
Under Assumptions~\ref{asm_LocalLinear},~\ref{asm_BoundC}, and~\ref{asm_OrthoTrajectory}, there exists $1-\beta_2<m<1$ such that if the step size satisfies $\frac{1-m}{\beta_1 \beta_2}<\eta<\frac{1}{\beta_1}$, \fedbug converges with a client discrepancy contraction ratio upper bounded by $\frac{1+m\cos^2\theta(1+\alpha)}{2}$, where $\theta$ is defined as in Theorem~\ref{thm_Fedavg}.
\end{theorem}

\textbf{Proof Sketch.} We defer the complete proof to Appendix~\ref{subsec:app_thm_convergence} and provide a proof sketch here. We start with the case of \fedavg. The goal of the proof is to show the discrepancy between clients contracts after each global round. To estimate the discrepancy term, we make use of geometric considerations and find that the client's discrepancy after local updates is a double projection of the original client discrepancy, where the projection angle is the difference between the local update direction and the axis of the last layer parameter. As for the convergence of \fedbug, we analyze the case where the first layer parameters are updated for one step, while the final layer parameter is kept frozen. With a proper step size, this update directly minimizes the client discrepancy and thus leads to a \textit{lower} client discrepancy when the clients reach their local minimum, compared with \fedavg. As a result, we can show that \fedbug exhibits a \textit{faster} convergence speed than \fedavg.

\begin{figure}[t]
\centering
\includegraphics[width=1\textwidth]{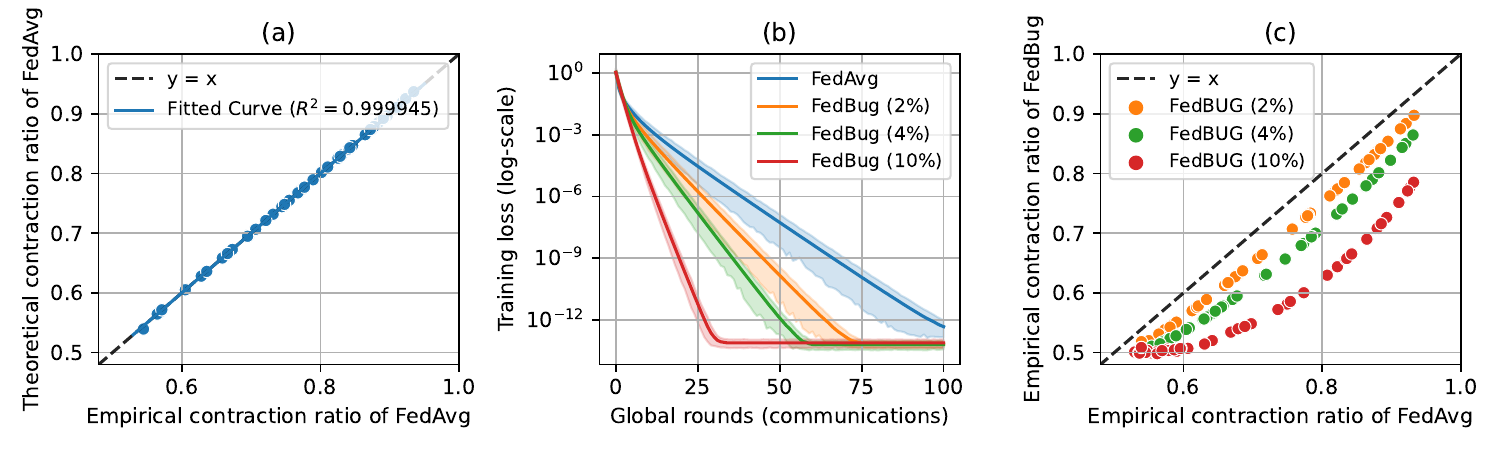}
\vspace{-5mm}
\caption{
\textbf{(a)} Alignment between theoretical and empirical client discrepancy contraction ratios.
\textbf{(b)} Comparisons of training loss curves of \fedavg and \fedbug. 
\textbf{(c)} Comparison of the contraction ratio of \fedavg and \fedbug.
Note that (a) is shown to confirm the validity of our assumptions and theorems, (b) is to validate our improved convergence over \fedavg, and (c) is to further verify that \fedbug achieves reduced client discrepancy when comparing to \fedavg.
}
\label{fig:thm_experiments}
\end{figure}

\subsection{Empirical Validation of Convergence Rate}
\label{subsec:thm_exp}
We conduct simulation experiments based on our task and model setup to empirically verify our theorem and the underlying assumptions. Details over setup can be found in Appendix~\ref{subsec:app_thm_exp}. The results are presented in Figure~\ref{fig:thm_experiments}. In Fig.~\ref{fig:thm_experiments}(a), we observe a high correlation between the contraction ratio estimated theoretically using $\theta$ and the one obtained by following Definitions~\ref{def_one} and~\ref{def_two} in \fedavg. This supports the validity of our assumptions and theorems. Based on the training loss curves shown in Fig.~\ref{fig:thm_experiments}(b), we see that \fedbug converges faster than \fedavg, confirming its superior convergence properties. As for Fig.~\ref{fig:thm_experiments}(c), we compare the empirical client discrepancy contraction ratios between \fedavg (x-axis) and \fedbug (y-axis). Given a contraction ratio of \fedavg, further decreased ratios can be observed for \fedbug with varying GU ratios. These simulation experiments provide empirical evidence supporting our theoretical analysis and verify the improved performance of \fedbug compared to \fedavg.

\section{Experiments}

We present an extensive evaluation of \fedbug across various FL scenarios. We provide a brief overview of the datasets and models, with a detailed description of our setup available in the Appendix. The code is written in PyTorch and executed on a single GPU, either an NVIDIA 3090 or V100. All experiments are performed with four distinct random seeds. Experimental details along with the ablation study are deferred to Appendices~\ref{subsec:app_exp_implement} and~\ref{subsec:app_exp_results}, respectively.

\subsection{Experimental Setup}
\textbf{Datasets.} We utilized benchmark datasets that follow the same train/test splits as previous works~\cite{mcmahan2017communication, li2020federated, acar2021federated}. These datasets include CIFAR-10, CIFAR-100, and Tiny-ImageNet. We randomly assigned data to the clients for the IID label distribution split~\cite{mcmahan2017communication, acar2021federated}. As for the non-IID label distribution, we followed the Dirichlet distribution Dir($\alpha$), as in~\cite{yurochkin2019bayesian, acar2021federated}. Here, $\alpha$ is a concentration parameter, with a smaller $\alpha$ indicating stronger data heterogeneity. When $\alpha$ equals $\infty$, the setting is homogeneous. We set $\alpha$ to $0.3$ for CIFAR-10 and CIFAR-100, and $0.5$ for Tiny-ImageNet.

\textbf{Models.} For standard CNN, we employ a standard convolutional neural network, similar to~\cite{mcmahan2017communication, acar2021federated}, consisting of two (three) convolutional layers followed by three fully connected layers for CIFAR-10 and CIFAR-100 (Tiny-ImageNet) dataset. For ResNet-18 and ResNet-34~\cite{he2016deep}, we change the batch normalization to group normalization~\cite{acar2021federated, hosseini2021federated, yu2021fed2, hyeon2021fedpara, hsieh2020non}.

\subsection{Experimental Results}

\definecolor{hl}{gray}{0.9}
\renewcommand{\arraystretch}{0.93}{
\setlength{\tabcolsep}{1.35mm}{
\begin{table*}[t]
\centering
\resizebox{1.0\textwidth}{!}{
\begin{tabular}{lcccccccccccccc}
\toprule
\multirow{4.5}{*}{Method} 
& \multicolumn{11}{c}{\textbf{CIFAR-10} (\# clients: $100$; client participation rate: $\textbf{1\%}$)} \\
\cmidrule{2-12}
& \multicolumn{5}{c}{IID label distribution ($\alpha=\infty$)} 
&& \multicolumn{5}{c}{Non-IID label distribution ($\alpha=0.3$)} \\
\cmidrule{2-6}\cmidrule{8-12}
& FedAvg & FedProx & FedDyn & FedExp & FedDecorr && FedAvg & FedProx & FedDyn & FedExp & FedDecorr \\
\midrule
Vanilla & 78.61 & 79.65 & 80.62 & 80.03 & 80.15 && 76.13 & 77.06 & 78.17 & 77.14 & 76.80 \\
FedBug (10\%) & \cellcolor{hl}{80.22} & \cellcolor{hl}{79.92} & \cellcolor{hl}{81.02} & \cellcolor{hl}{80.74} & \cellcolor{hl}{79.96} && \cellcolor{hl}{\textbf{\textbf{77.98}}} & \cellcolor{hl}{77.89} & \cellcolor{hl}{78.78} & \cellcolor{hl}{\textbf{77.94}} & \cellcolor{hl}{\textbf{77.58}} \\
FedBug (20\%) & \cellcolor{hl}{\textbf{80.63}} & \cellcolor{hl}{80.31} & \cellcolor{hl}{\textbf{81.47}} & \cellcolor{hl}{\textbf{81.00}} & \cellcolor{hl}{80.19} && \cellcolor{hl}{77.68} & \cellcolor{hl}{77.67} & \cellcolor{hl}{\textbf{78.80}} & \cellcolor{hl}{77.73} & \cellcolor{hl}{77.22} \\
FedBug (40\%) & \cellcolor{hl}{80.43} & \cellcolor{hl}{\textbf{80.34}} & \cellcolor{hl}{81.06} & \cellcolor{hl}{80.70} & \cellcolor{hl}{\textbf{80.27}} && \cellcolor{hl}{77.54} & \cellcolor{hl}{77.54} & \cellcolor{hl}{78.36} & \cellcolor{hl}{77.45} & \cellcolor{hl}{77.31} \\
\midrule
\noalign{\vskip 1mm}    
\midrule
\multirow{4.5}{*}{Method} 
& \multicolumn{11}{c}{\textbf{CIFAR-10} (\# clients: $100$; client participation rate: $\textbf{10\%}$)} \\
\cmidrule{2-12}
& \multicolumn{5}{c}{IID label distribution ($\alpha=\infty$)} 
&& \multicolumn{5}{c}{Non-IID label distribution ($\alpha=0.3$)} \\
\cmidrule{2-6}\cmidrule{8-12}
& FedAvg & FedProx & FedDyn & FedExp & FedDecorr && FedAvg & FedProx & FedDyn & FedExp & FedDecorr \\
\midrule
Vanilla & 80.06 & 79.86 & 82.62 & 79.95 & 80.01 && 78.11 & 77.67 & 81.15 & 77.74 & 78.12 \\
FedBug (10\%) & \cellcolor{hl}{80.57} & \cellcolor{hl}{80.51} & \cellcolor{hl}{82.89} & \cellcolor{hl}{80.69} & \cellcolor{hl}{80.66} && \cellcolor{hl}{78.43} & \cellcolor{hl}{78.58} & \cellcolor{hl}{81.21} & \cellcolor{hl}{78.27} & \cellcolor{hl}{78.65} \\
FedBug (20\%) & \cellcolor{hl}{80.66} & \cellcolor{hl}{80.55} & \cellcolor{hl}{83.27} & \cellcolor{hl}{81.10} & \cellcolor{hl}{81.31} && \cellcolor{hl}{78.86} & \cellcolor{hl}{78.77} & \cellcolor{hl}{81.48} & \cellcolor{hl}{78.86} & \cellcolor{hl}{78.67} \\
FedBug (40\%) & \cellcolor{hl}{\textbf{80.91}} & \cellcolor{hl}{\textbf{81.02}} & \cellcolor{hl}{\textbf{83.60}} & \cellcolor{hl}{\textbf{81.26}} & \cellcolor{hl}{\textbf{81.37}} && \cellcolor{hl}{\textbf{79.18}} & \cellcolor{hl}{\textbf{79.28}} & \cellcolor{hl}{\textbf{81.80}} & \cellcolor{hl}{\textbf{79.05}} & \cellcolor{hl}{\textbf{79.20}} \\
\bottomrule
\end{tabular}
}
\vspace{-1mm}
\caption{\textbf{Experiments on CIFAR-10 with standard CNN.} We conduct experiments with different client participation rates ($1\%$ and $10\%$), degrees of heterogeneity ($\alpha\in \{0.3, \infty\}$), and combinations of FL algorithms. Test accuracy (\%) is averaged over four seeds in all experiments. Bold font indicates the highest accuracy in each column, while the gray background highlights our framework.}
\label{tbl_c10}
\end{table*}
}
}
\renewcommand{\arraystretch}{0.93}{
\setlength{\tabcolsep}{1.35mm}{
\begin{table*}[t]
\centering
\resizebox{1.0\textwidth}{!}{
\begin{tabular}{lcccccccccccccc}
\toprule
\multirow{4.5}{*}{Method} 
& \multicolumn{11}{c}{\textbf{CIFAR-100} (\# clients: $100$; client participation rate: $\textbf{1\%}$)} \\
\cmidrule{2-12}
& \multicolumn{5}{c}{IID label distribution ($\alpha=\infty$)} 
&& \multicolumn{5}{c}{Non-IID label distribution ($\alpha=0.3$)} \\
\cmidrule{2-6}\cmidrule{8-12}
& FedAvg & FedProx & FedDyn & FedExp & FedDecorr && FedAvg & FedProx & FedDyn & FedExp & FedDecorr \\
\midrule
Vanilla & 39.33 & 40.19 & 49.61 & 39.51 & 42.71 && 40.43 & 40.46 & 48.69 & 41.32 & 43.81 \\
FedBug (10\%) & \cellcolor{hl}{41.56} & \cellcolor{hl}{41.17} & \cellcolor{hl}{51.18} & \cellcolor{hl}{42.26} & \cellcolor{hl}{43.38} && \cellcolor{hl}{42.99} & \cellcolor{hl}{42.27} & \cellcolor{hl}{50.02} & \cellcolor{hl}{42.86} & \cellcolor{hl}{44.53} \\
FedBug (20\%) & \cellcolor{hl}{43.28} & \cellcolor{hl}{43.49} & \cellcolor{hl}{51.49} & \cellcolor{hl}{43.51} & \cellcolor{hl}{44.14} && \cellcolor{hl}{43.47} & \cellcolor{hl}{43.34} & \cellcolor{hl}{50.58} & \cellcolor{hl}{43.67} & \cellcolor{hl}{44.75} \\
FedBug (40\%) & \cellcolor{hl}{\textbf{45.59}} & \cellcolor{hl}{\textbf{45.33}} & \cellcolor{hl}{\textbf{51.91}} & \cellcolor{hl}{\textbf{45.83}} & \cellcolor{hl}{\textbf{44.87}} && \cellcolor{hl}{\textbf{45.35}} & \cellcolor{hl}{\textbf{44.61}} & \cellcolor{hl}{\textbf{51.00}} & \cellcolor{hl}{\textbf{45.03}} & \cellcolor{hl}{\textbf{44.78}} \\
\midrule
\noalign{\vskip 1mm}    
\midrule
\multirow{4.5}{*}{Method} 
& \multicolumn{11}{c}{\textbf{CIFAR-100} (\# clients: $100$; client participation rate: $\textbf{10\%}$)} \\
\cmidrule{2-12}
& \multicolumn{5}{c}{IID label distribution ($\alpha=\infty$)} 
&& \multicolumn{5}{c}{Non-IID label distribution ($\alpha=0.3$)} \\
\cmidrule{2-6}\cmidrule{8-12}
& FedAvg & FedProx & FedDyn & FedExp & FedDecorr && FedAvg & FedProx & FedDyn & FedExp & FedDecorr \\
\midrule
Vanilla & 44.45 & 44.64 & 46.62 & 46.01 & 45.14 && 44.05 & 43.58 & 45.01 & 44.89 & 43.28 \\
FedBug (10\%) & \cellcolor{hl}{47.51} & \cellcolor{hl}{46.85} & \cellcolor{hl}{47.89} & \cellcolor{hl}{48.68} & \cellcolor{hl}{46.12} && \cellcolor{hl}{47.38} & \cellcolor{hl}{47.07} & \cellcolor{hl}{47.03} & \cellcolor{hl}{47.83} & \cellcolor{hl}{\textbf{44.27}} \\
FedBug (20\%) & \cellcolor{hl}{48.77} & \cellcolor{hl}{48.35} & \cellcolor{hl}{48.20} & \cellcolor{hl}{49.54} & \cellcolor{hl}{\textbf{46.30}} && \cellcolor{hl}{\textbf{47.92}} & \cellcolor{hl}{\textbf{47.87}} & \cellcolor{hl}{\textbf{47.63}} & \cellcolor{hl}{\textbf{48.16}} & \cellcolor{hl}{43.99} \\
FedBug (40\%) & \cellcolor{hl}{\textbf{49.51}} & \cellcolor{hl}{\textbf{49.42}} & \cellcolor{hl}{\textbf{48.80}} & \cellcolor{hl}{\textbf{50.03}} & \cellcolor{hl}{46.29} && \cellcolor{hl}{47.80} & \cellcolor{hl}{47.67} & \cellcolor{hl}{47.32} & \cellcolor{hl}{47.87} & \cellcolor{hl}{44.02} \\
\bottomrule
\end{tabular}
}
\vspace{-1mm}
\caption{\textbf{Experiments on CIFAR-100 with standard CNN.} We conduct experiments with different client participation rates ($1\%$ and $10\%$), degrees of heterogeneity ($\alpha\in \{0.3, \infty\}$), and FL algorithms. Note that improvements on CIFAR-100 are more remarkable than those on CIFAR-10.
}
\label{tbl_c100}
\end{table*}
}
}

\textbf{Improved Performance and High Compatibility With Various FL Algorithms.}
The main results on CIFAR-10, CIFAR-100, and Tiny-ImageNet dataset are presented in Tables~\ref{tbl_c10},~\ref{tbl_c100}, and~\ref{tbl_tin_cnn}. 
These results highlight the high compatibility of the \fedbug framework when combined with different FL algorithms. 
Additionally, we consistently observe that the \fedbug framework outperforms the vanilla training framework, even when the gradual unfreezing stage comprises only ten percent of the training process. This indicates the effectiveness and efficiency of the \fedbug approach in improving model performance. 
Furthermore, the \fedbug training framework demonstrates a consistent synergistic effect across five different FL algorithms, two client participation levels, and both IID and non-IID label distributions. This demonstrates the broad applicability of our proposed framework in combination with existing FL training algorithms and experimental setups.

\definecolor{hl}{gray}{0.9}
\renewcommand{\arraystretch}{0.93}{
\setlength{\tabcolsep}{1.35mm}{
\begin{table*}[t]
\centering
\resizebox{1.0\textwidth}{!}{
\begin{tabular}{lcccccccccccccc}
\toprule
\multirow{4.5}{*}{Method} 
& \multicolumn{11}{c}{\textbf{Tiny-ImageNet} (\# clients: $10$; client participation rate: $\textbf{10\%}$)} \\
\cmidrule{2-12}
& \multicolumn{5}{c}{IID label distribution ($\alpha=\infty$)} 
&& \multicolumn{5}{c}{Non-IID label distribution ($\alpha=0.5$)} \\
\cmidrule{2-6}\cmidrule{8-12}
& FedAvg & FedProx & FedDyn & FedExp & FedDecorr && FedAvg & FedProx & FedDyn & FedExp & FedDecorr \\
\midrule
Vanilla & 28.40 & 28.41 & 29.84 & 29.39 & 28.15 && 26.40 & 26.76 & 28.46 & 27.44 & 26.51 \\
FedBug (12\%) & \cellcolor{hl}{28.87} & \cellcolor{hl}{28.92} & \cellcolor{hl}{30.18} & \cellcolor{hl}{29.76} & \cellcolor{hl}{28.77} && \cellcolor{hl}{27.23} & \cellcolor{hl}{27.34} & \cellcolor{hl}{28.94} & \cellcolor{hl}{27.93} & \cellcolor{hl}{26.76} \\
FedBug (24\%) & \cellcolor{hl}{29.04} & \cellcolor{hl}{29.35} & \cellcolor{hl}{30.39} & \cellcolor{hl}{\textbf{30.63}} & \cellcolor{hl}{28.70} && \cellcolor{hl}{\textbf{27.78}} & \cellcolor{hl}{27.25} & \cellcolor{hl}{29.12} & \cellcolor{hl}{28.37} & \cellcolor{hl}{26.93} \\
FedBug (48\%) & \cellcolor{hl}{\textbf{29.58}} & \cellcolor{hl}{\textbf{29.67}} & \cellcolor{hl}{\textbf{27.78}} & \cellcolor{hl}{30.42} & \cellcolor{hl}{\textbf{29.29}} && \cellcolor{hl}{27.65} & \cellcolor{hl}{\textbf{27.83}} & \cellcolor{hl}{\textbf{29.45}} & \cellcolor{hl}{\textbf{28.44}} & \cellcolor{hl}{\textbf{27.75}} \\

\midrule
\noalign{\vskip 1mm}    
\midrule
\multirow{4.5}{*}{Method} 
& \multicolumn{11}{c}{\textbf{Tiny-ImageNet} (\# clients: $10$; client participation rate: $\textbf{30\%}$)} \\
\cmidrule{2-12}
& \multicolumn{5}{c}{IID label distribution ($\alpha=\infty$)} 
&& \multicolumn{5}{c}{Non-IID label distribution ($\alpha=0.5$)} \\
\cmidrule{2-6}\cmidrule{8-12}
& FedAvg & FedProx & FedDyn & FedExp & FedDecorr && FedAvg & FedProx & FedDyn & FedExp & FedDecorr \\
Vanilla & 26.48 & 26.62 & 31.35 & 26.79 & 28.96 && 25.09 & 25.40 & 29.78 & 24.98 & 26.57 \\
FedBug (12\%) & \cellcolor{hl}{26.93} & \cellcolor{hl}{26.95} & \cellcolor{hl}{31.31} & \cellcolor{hl}{26.79} & \cellcolor{hl}{29.50} && \cellcolor{hl}{25.69} & \cellcolor{hl}{25.71} & \cellcolor{hl}{29.95} & \cellcolor{hl}{25.72} & \cellcolor{hl}{26.52} \\
FedBug (24\%) & \cellcolor{hl}{26.94} & \cellcolor{hl}{27.15} & \cellcolor{hl}{\textbf{31.54}} & \cellcolor{hl}{27.10} & \cellcolor{hl}{29.43} && \cellcolor{hl}{25.82} & \cellcolor{hl}{26.02} & \cellcolor{hl}{30.05} & \cellcolor{hl}{25.62} & \cellcolor{hl}{27.49} \\
FedBug (48\%) & \cellcolor{hl}{\textbf{27.64}} & \cellcolor{hl}{\textbf{27.84}} & \cellcolor{hl}{31.44} & \cellcolor{hl}{\textbf{27.76}} & \cellcolor{hl}{\textbf{29.60}} && \cellcolor{hl}{\textbf{26.36}} & \cellcolor{hl}{\textbf{26.81}} & \cellcolor{hl}{\textbf{30.71}} & \cellcolor{hl}{\textbf{26.33}} & \cellcolor{hl}{\textbf{27.82}} \\
\bottomrule
\end{tabular}
}
\vspace{-1mm}
\caption{\textbf{Experiments on Tiny-ImageNet of $10$ clients with standard CNN.} We conduct experiments with different client participation rates ($10\%$ and $30\%$), degrees of heterogeneity ($\alpha\in \{0.3, \infty\}$), and combinations of FL algorithms. 
}
\vspace{-1mm}
\label{tbl_tin_cnn}
\end{table*}
}
}

\textbf{Applicability of \fedbug on ResNet.} When applying \fedbug to larger models like ResNet, a natural question arises: What should be the smallest unit to unfreeze during the GU stage --- a ResNet Module or a residual block? Since both ResNet-18 and ResNet-34 can be seen as having four ResNet Modules or consisting of eight and sixteen residual blocks, respectively, a robust framework should remain agnostic to the unit definition. Therefore, we evaluate the adaptability of \fedbug using two distinct unfreezing strategies: (1) progressively unfreezing one ResNet Module at a time, and (2) progressively unfreezing one residual block at a time. Notably, we capitalize the term ResNet "Module" to differentiate it from the general module mentioned in Algorithm~\ref{alg:fedbug}. 

The experimental results, presented in Table~\ref{tbl_resnet}, consistently demonstrate the superiority of the \fedbug framework over the vanilla training framework across both unfreezing strategies and different label distributions. Interestingly, we observe that both strategies perform comparably well on the datasets, indicating the effectiveness of \fedbug.
\definecolor{hl}{gray}{0.9}
\renewcommand{\arraystretch}{0.94}{
\setlength{\tabcolsep}{1.5mm}{
\begin{table*}[t]
\centering
\resizebox{1.0\textwidth}{!}{
\begin{tabular}{lcccccccccccccc}
\toprule
\multirow{5.5}{*}{Method} 
& \multicolumn{11}{c}{\textbf{CIFAR-100} (\# clients: $10$; client participation rate: $10\%$)} \\
\cmidrule{2-12}
& \multicolumn{5}{c}{IID label Distribution ($\alpha=\infty$)} 
&& \multicolumn{5}{c}{Non-IID label Distribution ($\alpha=0.3$)} \\
\cmidrule{2-6} \cmidrule{8-12}
& \multicolumn{2}{c}{ResNet-18} 
&& \multicolumn{2}{c}{ResNet-34} 
&& \multicolumn{2}{c}{ResNet-18}
&& \multicolumn{2}{c}{ResNet-34} 
\\
\cmidrule{2-3} \cmidrule{5-6} \cmidrule{8-9} \cmidrule{11-12}
& Module(4) & Block(8) && Module(4) & Block(16)
&& Module(4) & Block(8) && Module(4) & Block(16)\\
\midrule
Vanilla & 52.59 & 52.59 && 52.64 & 52.64 && 49.04 & 49.04 && 48.69 & 48.69 \\
FedBug (20\%) & \cellcolor{hl}{53.25} & \cellcolor{hl}{53.05} 
&& \cellcolor{hl}{53.01} & \cellcolor{hl}{53.42} 
&& \cellcolor{hl}{\textbf{49.70}} & \cellcolor{hl}{49.64} 
&& \cellcolor{hl}{49.20} & \cellcolor{hl}{49.17} \\
FedBug (40\%) & \cellcolor{hl}{\textbf{53.65}} & \cellcolor{hl}{\textbf{53.49}} 
&& \cellcolor{hl}{\textbf{53.56}} & \cellcolor{hl}{\textbf{53.56}} 
&& \cellcolor{hl}{49.36} & \cellcolor{hl}{\textbf{49.69}} 
&& \cellcolor{hl}{\textbf{49.37}} & \cellcolor{hl}{49.33} \\
\midrule
\noalign{\vskip 1mm}    
\midrule
\multirow{5.5}{*}{Method} 
& \multicolumn{11}{c}{\textbf{Tiny-ImageNet} (\# clients: $10$; client participation rate: $10\%$)} \\
\cmidrule{2-12}
& \multicolumn{5}{c}{IID label Distribution ($\alpha=\infty$)} 
&& \multicolumn{5}{c}{Non-IID label Distribution ($\alpha=0.5$)} \\
\cmidrule{2-6} \cmidrule{8-12}
& \multicolumn{2}{c}{ResNet-18} 
&& \multicolumn{2}{c}{ResNet-34} 
&& \multicolumn{2}{c}{ResNet-18}
&& \multicolumn{2}{c}{ResNet-34} 
\\
\cmidrule{2-3} \cmidrule{5-6} \cmidrule{8-9} \cmidrule{11-12}
& Module(4) & Block(8) && Module(4) & Block(16)
&& Module(4) & Block(8) && Module(4) & Block(16)\\
\midrule
Vanilla & 33.88 & 33.88 
&& 33.22 & 33.22
&& 31.91 & 31.91 
&& 31.53 & 31.53 \\
FedBug (20\%) & \cellcolor{hl}{34.25} & \cellcolor{hl}{34.31} 
&& \cellcolor{hl}{34.28} & \cellcolor{hl}{34.36} 
&& \cellcolor{hl}{32.29} & \cellcolor{hl}{32.32} 
&& \cellcolor{hl}{32.33} & \cellcolor{hl}{32.31} \\
FedBug (40\%) & \cellcolor{hl}{\textbf{35.28}} & \cellcolor{hl}{\textbf{35.17}} 
&& \cellcolor{hl}{\textbf{35.12}} & \cellcolor{hl}{\textbf{35.10}} 
&& \cellcolor{hl}{\textbf{32.86}} & \cellcolor{hl}{\textbf{33.47}} 
&& \cellcolor{hl}{\textbf{33.20}} & \cellcolor{hl}{\textbf{33.36}} \\
\bottomrule
\end{tabular}
}
\vspace{-1mm}
\caption{\textbf{Experiments on ResNet with module-wise and block-wise unfreezing strategies.} Note that two unfreezing strategies are considered: (1) unfreezing one ResNet \textit{Module} at a time and (2) unfreezing one residual \textit{Block} at a time. Module (4) indicates that the model consists of four ResNet Modules, while Block (16) signifies the model consists of sixteen residual Blocks. Consistent improvements of \fedbug with both unfreezing strategies can be observed.}
\label{tbl_resnet}
\end{table*}
}
}


\textbf{Impact of Gradual Unfreezing Percentage.} 
We investigate the impact of the percentage of the GU stage in CIFAR-100 and Tiny-ImageNet in the standard CNN model. The baseline framework is \fedavg, represented by a percentage of $0\%$. The results are shown in Figure~\ref{fig:exp_percentage_gu}. 
Our experiments reveal consistent improvements in test accuracy even with longer GU ratios. 
Notably, allocating a larger percentage of the training period to GU leads to the top layers receiving less training. 
Consequently, the test accuracy does not necessarily increase monotonically with the GU ratio. For instance, when training a five-layer model with \fedbug using a 100\% GU stage ratio, the penultimate linear layer and the classifier are trained for only 40\% and 20\% of the total training iterations, respectively. 
In this extreme scenario, \fedbug not only saves considerable training time but also provides improvement. 
These results highlight the robustness of \fedbug to the GU ratio and suggest a small portion of the training time for gradual unfreezing may readily yield favorable results.

\begin{figure}[ht]
\centering
\includegraphics[width=0.9\textwidth]{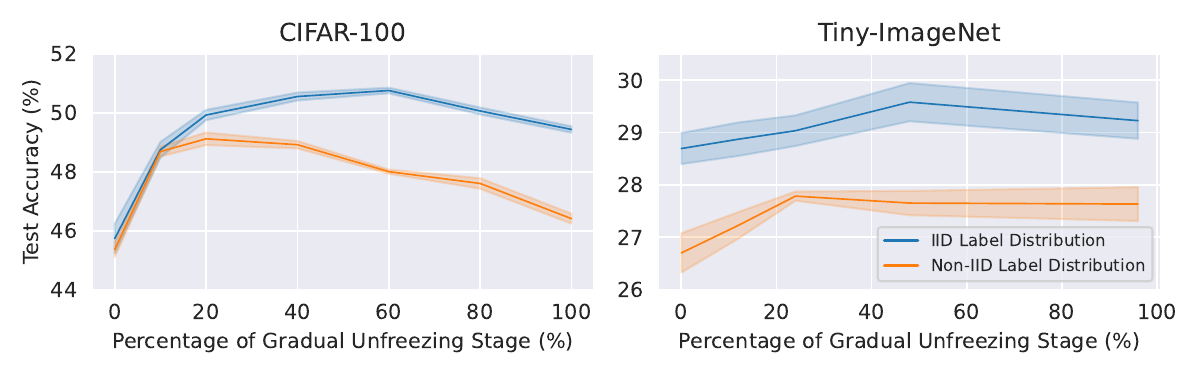}
\vspace{-2mm}
\caption{
    \textbf{Impact of Gradual Unfreezing Percentage.} 
    We investigate the effect of the percentage of the GU stage in CIFAR-100 and Tiny-ImageNet. The baseline framework is \fedavg (i.e., GU percentage of $0.0$). Our experiments reveal consistent improvements in test accuracy even with longer GU ratios, where the top layers receive considerably less training compared to \fedavg.
    }
\label{fig:exp_percentage_gu}
\end{figure}
\section{Conclusion}
In this work, we introduce \fedbug, a novel FL framework designed to mitigate client drift. 
By leveraging model parameters as anchors, \fedbug aligns clients while improving learning efficiency.
We perform theoretical analysis in an over-parameterized setting, revealing that \fedbug achieves a faster convergence rate compared to the widely adopted \fedavg framework.  
To empirically validate the effectiveness, we conduct extensive experiments on various datasets, training conditions, and network architectures, consistently demonstrating the superiority and compatibility of \fedbug.
Overall, our contributions include the introduction of a novel FL framework, theoretical analysis, and comprehensive empirical validation, highlighting the broad potential and applicability of FedBug.

\textbf{Limitations.} 
Our analysis is limited to a specific scenario involving two clients, a two-layer linear network, and an orthogonal dataset distribution. 
We do not address the challenges related to multiple clients, non-IID label distributions, non-orthogonal dataset distributions, or general convex functions, leaving these as potential areas for future research. In Appendix~\ref{sec:thm_disscussion}, we provide a detailed discussion on the three assumptions made in order to make the problem more tractable.

\textbf{Broader Impact.} The \fedbug framework has the potential for wide-ranging impact in areas where FL intersects with natural language processing, reinforcement learning, autonomous driving, personalized computing, and more. The uniqueness of the bottom-up gradual unfreezing concept may also open avenues in general distributed computing.


\bibliographystyle{plainnat}
\bibliography{reference}

\clearpage
\appendix

\section{Experiments}

\subsection{Implementation Details}
\label{subsec:app_exp_implement}
\textbf{Standard CNN.} Our code is based on~\cite{acar2021federated} and we extend it to include the Tiny-ImageNet dataset, the latest proposed FedDecorr~\cite{shi2022towards} and FedExP~\cite{jhunjhunwala2023fedexp} algorithms. 
We use Stochastic Gradient Descent (SGD) optimizer and a cross entropy loss function, with a learning rate of 0.1 and weight decay of 0.001. We use a batch size of 50 and perform horizontal flipping for training data augmentation on all datasets, while adding cropping augmentation on CIFAR-10 and CIFAR-100~\cite{acar2021federated}. 
For the training epochs, we run 300, 500, and 100 global rounds (communication rounds) with 10, 5, and 5 local epochs in CIFAR-10, CIFAR-100, and Tiny-ImageNet, respectively, as suggested in~\cite{acar2021federated, shi2022towards}. In the CIFAR-10 and CIFAR-100 (Tiny-ImageNet) datasets, we consider 100 (10) participants with participation rates of 1\% and 10\% (10\% and 30\%) at each global round, respectively~\cite{acar2021federated}. 1\% participation rate means that each client had a 1\% chance of being selected to join in a global round.

\textbf{ResNet.} For the ResNet architecture, we focus on CIFAR-100 and Tiny-ImageNet datasets. We conducted 100 global rounds with 5 local epochs on CIFAR-100 and 3 local epochs on Tiny-ImageNet~\cite{acar2021federated, shi2022towards}. In both datasets, we work with 10 participants with a 10\% client participation rate.
Notably, ResNet-18 and ResNet-34 can be treated as having either four modules each, or eight and sixteen residual blocks, respectively. To test the applicability of \fedbug, we consider two strategies: (1) unfreezing one ResNet module at a time, or (2) unfreezing one residual block at a time. The first strategy corresponds to a scenario where $M=4$, while the second corresponds to $M=8$ for ResNet-18 and $M=16$ for ResNet-34.

\textbf{Hyperparameters.} Our use of hyperparameters is similar to~\cite{acar2021federated}, where $\mu = 0.0001$ for FedProx~\cite{li2020federated}, $alpha = 0.01$ for FedDyn~\cite{acar2021federated}. We use $\beta = 0.01$ for FedDecorr~\cite{shi2022towards}, while FedExp~\cite{jhunjhunwala2023fedexp} does not require additional hyperparameters. 

\subsection{Experimental Results}
\label{subsec:app_exp_results}

\textbf{Ablation Study: Comparative Analysis of Freezing Strategies.} 
To compare the impact of different freezing strategies, we include the following methods that are closely related to ours: (1) Top-Down Gradual Unfreezing, which has been adopted in recent NLP literature for model fine-tuning~\cite{howard2018universal, mukherjee2019distilling, raffel2020exploring, liu2023improving} and fine-tuning the model from the output layer to the input layer; (2) Fixing the Last Layer throughout training, known as FedBABU~\cite{oh2021fedbabu}; and (3) Fixing the Last Two Layers, i.e., the classifier and the penultimate layer. It is worth noting that in the work of [\cite{howard2018universal}], they propose the use of Universal Language Model Fine-tuning (ULMFiT) for fine-tuning large NLP models for transfer learning, where the method incorporates top-down gradual unfreezing. The rationale behind top-down gradual unfreezing is that the last layer contains the least general knowledge and is more specialized for the original task. By employing top-down gradual unfreezing, the bottom layers undergo fewer changes, mitigating the risk of catastrophic forgetting and facilitating adaptation to the new task. While ULMFiT focuses on specific task adaptation for fine-tuning, the aim of \fedbug is to increase cross-client alignment through parameter anchors.

For our baseline, we utilize \fedavg and \fedbug $(20\%)$, referred to as "Vanilla" and "FedBug: Bottom-Up GU", respectively.
The results on CIFAR-10, CIFAR-100, and Tiny-ImageNet are shown in Figure~\ref{fig:exp_ablation_fixation}. We observe that all the alternative unfreezing or freezing frameworks perform worse than \fedavg, while our proposed method consistently outperforms them.
This ablation study highlights the importance of freezing strategies in achieving optimal test performance and provides compelling evidence supporting the intuition that gradually unfreezing layers from the bottom up leads to improved performance.


\begin{figure}[ht]
\centering
\includegraphics[width=1.\textwidth]{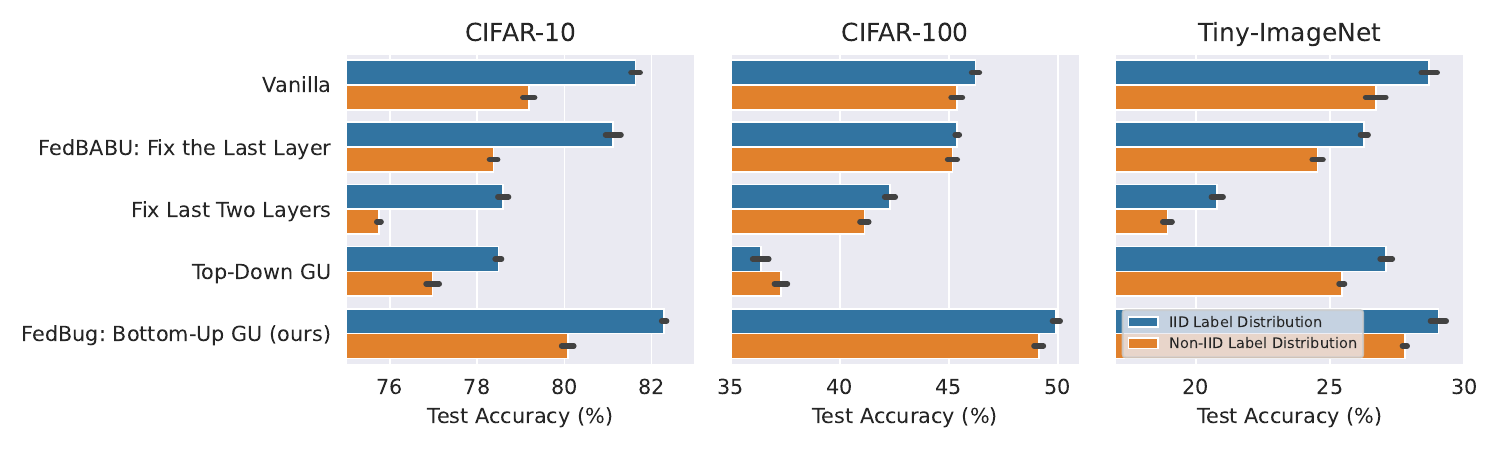}
\caption{
    \textbf{Impact of Gradual Unfreezing Percentage.} 
    }
\label{fig:exp_ablation_fixation}
\end{figure}

\textbf{Ablation Study: Number of Clients.} 
To compare the impact of different numbers of clients, we conducted an ablation study using a consistent client participation rate of 10\% for each setting. We utilized the CIFAR-100 dataset with ResNet-18 and employed the ResNet Module-wise unfreezing strategy. We utilize \fedbug (50\%) for this ablation study.
The experimental results were averaged over four random seeds. The results are summarized in Table~\ref{tbl_c100_nc}, revealing that even with a large number of clients, the \fedbug framework consistently improves testing accuracy.

\definecolor{hl}{gray}{0.9}
\renewcommand{\arraystretch}{0.99}{
\setlength{\tabcolsep}{5mm}{
\begin{table*}[t]
\centering
\resizebox{1.0\textwidth}{!}{
\begin{tabular}{lcccccccccccc}
\toprule
\multirow{4.5}{*}{Method} 
& \multicolumn{7}{c}{
\textbf{CIFAR-100} (ResNet-18; client participation rate: 10\%)} \\
\cmidrule{2-8}
&  \multicolumn{3}{c}{IID label distribution ($\alpha=\infty$)} 
&& \multicolumn{3}{c}{Non-IID label distribution ($\alpha=0.3$)} \\
\cmidrule{2-8}
&  \multicolumn{3}{c}{\# Clients} 
&& \multicolumn{3}{c}{\# Clients} \\
&   10 & 50 & 500 
&&  10 & 50 & 500 \\
\midrule
Vanilla 
&  52.55 & 42.40 
& 19.45 
&& 49.05 & 40.65 
& 19.32 \\
FedBug
& \cellcolor{hl}{\textbf{53.59}} & \cellcolor{hl}{\textbf{44.44}} 
& \cellcolor{hl}{\textbf{21.79}} 
&& \cellcolor{hl}{\textbf{49.93}} & \cellcolor{hl}{\textbf{41.32}} 
& \cellcolor{hl}{\textbf{19.56}} \\
\bottomrule
\end{tabular}
}
\vspace{-1mm}
\caption{\textbf{Experiments on CIFAR-100 with varying number of clients on ResNet-18.}
}
\vspace{-1mm}
\label{tbl_c100_nc}
\end{table*}
}
}



\section{Full Theoretical Analysis}
\label{sec:app_thm}
In this section, we present a theoretical analysis of the FedAvg and FedBug algorithms in a two-layer linear network with an orthogonal dataset distribution. 

\subsection{Task Setting and Model Architecture}
\label{subsec:app_thm_setup}

\textbf{Task and Evaluation.}
We consider a FL regression task, with two clients denoted as $c_1$ and $c_2$. Each client has different regression data, specifically $\mathcal{T}_1=\{x_1=[1,0], y_1=1\}$ and $\mathcal{T}_2=\{x_2=[0,1], y_2=1\}$. The objective is to minimize the L2 loss, with client $c_1$ ($c_2$) minimizing $L_1=\normx{f(x_1)-y_1}$ ($L_2=\normx{f(x_2)-y_2}^2$), where $f$ denotes the model.

\textbf{Model Architecture.}
The model architecture is a two-layer linear network $f$ with two nodes $[a, b]$ in the first layer and one node $[v]$ in the second layer. Specifically, $f(x) = x[a,b]^\top c$. The task setup implies that client $\ca$ ($\cb$) aims to minimize $L_1=|av-1|^2$ ($L_2=|bv-1|^2$).

The chosen setup, which includes an orthogonal dataset distribution, regression tasks, and a two-layer neural network model, has been previously employed to study transfer learning with out-of-distribution datasets~\cite{Kumar2022fine} and domain generalization~\cite{lee2023surgical}. Specifically,~\cite{Kumar2022fine} demonstrates that fine-tuning on out-of-distribution data can lead to feature distortion on in-distribution data. Meanwhile,~\cite{lee2023surgical} focused on the case of fine-tuning only the pre-trained encoder, which outperforms full-layer fine-tuning in domain generalization when the target data is insufficient. Our work contributes to the existing literature by extending previous approaches to a federated learning setup, where two clients have orthogonal dataset distributions. Addressing this new scenario introduces unique challenges and considerations that have not been explored in previous works. 

\subsection{Preliminary: \fedavg and \fedbug}

We review \fedavg and its notations, as illustrated in Figure~\ref{fig:thm_illustration}. 
During the $i$-th global round, the server distributes parameters $[a^i,b^i,v^i]$ to the clients. For example, client $\ca$ receives the initial parameter $[a^i_{c_1, 0}, b^i_{c_1, 0}, v^i_{c_1, 0}] (= [a^i,b^i,v^i])$. 
Each client individually optimizes cost function using their own parameters. After $k$ local iterations, the parameters of client $\ca$ are updated to $[a^i_{c_1, k}, b^i_{c_1, k}, v^i_{c_1, k}]$, and upon achieving local convergence, they become $[a^i_{c_1, *}, b^i_{c_1, *}, v^i_{c_1, *}]$. 
Once local convergence is reached, clients send their learned parameters back to the server. The server then averages the received parameters to obtain the new parameters $[a^{i+1},b^{i+1},v^{i+1}]=\frac{1}{2}([a^i_{c_1, *}, b^i_{c_1, *}, v^i_{c_1, *}]+[a^i_{c_2, *}, b^i_{c_2, *}, v^i_{c_2, *}])$. 
We denote the solution obtained at global convergence as $[a^*, b^*, v^*]$. Notably, a solution must satisfy $a^*=b^*$ since $a^*v^*=1$ and $b^*v^*=1$.

The proposed \fedbug differs from \fedavg only in the client-side update step: Clients first freeze $v$ and update $[a,b]$ for a few local iterations. Afterwards, the client unfreezes the last layer parameters and performs gradient descent on all the parameters. 

\begin{figure}[ht]
\centering
\includegraphics[width=1\textwidth]{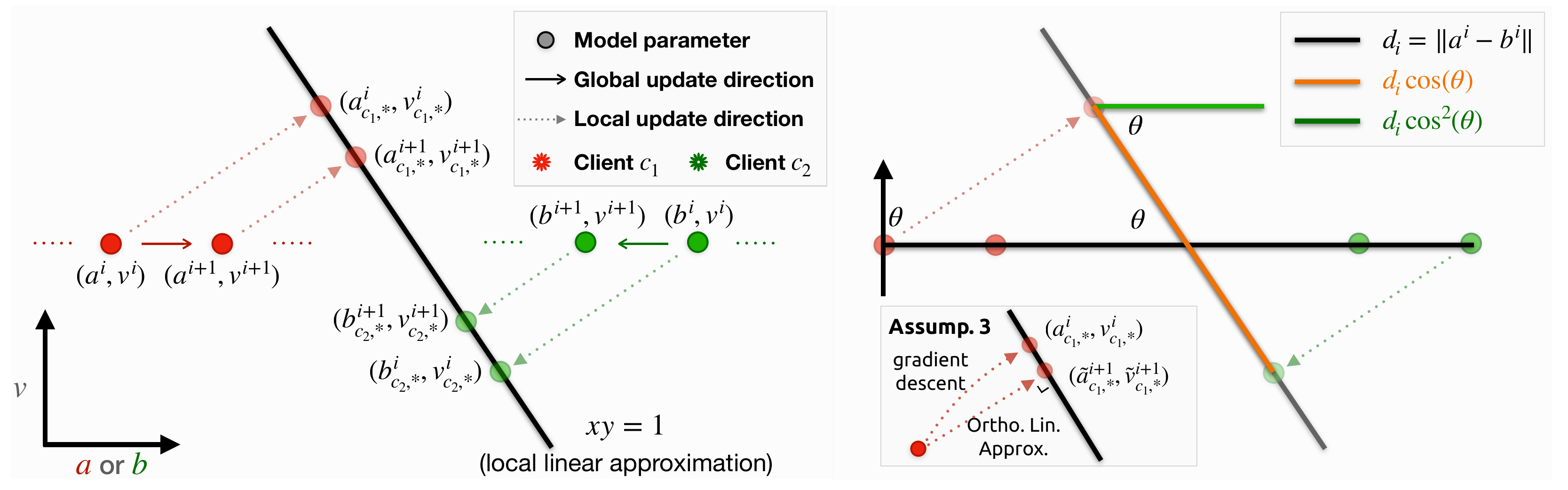}
\caption{
        \textbf{Left: Notation for our theoretical analysis.} The horizontal axis concurrently signifies the parameter axis for parameters $a$ and $b$. This approach enables us to visually explain the concept of client discrepancy, a term we define and utilize later.
        \textbf{Right: Visualization for Theorem~\ref{app_thm_Fedavg} and Assumption~\ref{app_asm_OrthoTrajectory}.} Different colors indicate the segments needed to calculate the client discrepancy contraction ratio. A higher alignment between the local update direction (represented by the dotted line) and the $v$ axis results in a larger contraction ratio. The notion behind Assumption~\ref{app_asm_OrthoTrajectory} is also visually illustrated.
}
\label{fig:thm_illustration}
\end{figure}

\subsection{The Learning Dynamics of Models Trained With \fedavg}
\label{subsec:app_thm_dynamics}

In this subsection, we provide an overview of the learning dynamics of models trained using the FedAvg algorithm. Specifically, at the client-side, if models take an infinitely small learning rate such that the optimization process becomes a gradient flow~\cite{gunasekar2017implicit, du2018algorithmic, du2018gradient, ye2023freeze}, the parameters' learning trajectory follows a hyperbolic track. Thus, each client's minimizers can be analytically derived due to the deterministic nature of the trajectory. 
At the server-side, the minimizers of the clients are averaged, leading to a discontinuity between each client's minimizers and the averaged parameters, presenting a challenge in discussing the continuous learning behavior of the model. To overcome this challenge, we introduce assumptions that enable us to better analyze the learning behavior.

\textbf{Local Iteration (Client Side).} 
In this section, we derive the local learning dynamics of the parameters for an infinitely small step size. We consider a gradient flow, implying that the model parameters are functions of time $t$. We denote the parameters of the model for client $c_1$ at time $t$ as $a(t)$, $b(t)$, and $v(t)$. A similar derivation applies for client $c_2$.

The gradient of the model at time $t$ can be defined as follows:
\begin{align*}
\frac{da(t)}{dt} & = -\eta\frac{\partial L_1}{\partial a(t)} = -\eta v(t)(a(t) v(t)-1),\\
\frac{db(t)}{dt} & = -\eta\frac{\partial L_1}{\partial b(t)} = 0,\\
\frac{dv(t)}{dt} & = -\eta\frac{\partial L_1}{\partial v(t)} = -\eta a(t)(a(t) v(t)-1).
\end{align*}
Subsequently, by multiplying $\frac{da(t)}{dt}$ by $a(t)$ and $\frac{dv(t)}{dt}$ by $v(t)$, we obtain:
\begin{align*}
a(t)\frac{da(t)}{dt} = v(t)\frac{dv(t)}{dt}.
\end{align*}
This equality leads to the following conservation law:
\begin{align*}
a(t)^2-v(t)^2 = a(0)^2-v(0)^2 = {a^i}^2-{v^i}^2.
\end{align*}
Assuming that ${a^i}^2-{v^i}^2$ is not zero, this equation represents a hyperbola in the $(a,v)$ plane, this suggests that the evolution of the parameters $a(t)$ and $v(t)$ follows a hyperbolic trajectory.

The local minimizers can be derived as follows. Recall that the learning trajectory at the client side is hyperbolic and the final solution must satisfy $av=1$. Using high school algebra, we obtain the following analytic solutions: 
\begin{align*}
a^i_{c_1, \infty} = a({\infty})
=\frac{1}{2}\sqrt{
            \frac{1}{2}
            (\sqrt{({a^i}^2-{v^i}^2)^2 + 4} + 
            ({a^i}^2-{v^i}^2))
            }, \\
v^i_{c_1, \infty} = v({\infty})
=\frac{1}{2}\sqrt{
            \frac{1}{2}
            (\sqrt{({a^i}^2-{v^i}^2)^2 + 4} - 
            ({a^i}^2-{v^i}^2))
            }.     
\end{align*}

\textbf{Global Round (Server Side).} This section examines the behavior of the averaged client parameters. According to the update rules of FedAvg, we have the following:
\begin{align*}
a^{i+1} 
& = \frac{1}{2}(a^i_{c_1, \infty}+a^i), \\
b^{i+1} 
& = \frac{1}{2}(b^i+b^i_{c_2, \infty}), \\ v^{i+1} 
& = \frac{1}{2}(v^i_{c_1, \infty}+v^i_{c_2, \infty}) \\
& = \frac{1}{4}(\sqrt{
            \frac{1}{2}
            (\sqrt{({a^i}^2-{v^i}^2)^2 + 4} - 
            ({a^i}^2-{v^i}^2))
            } \\
& \hspace{4.4mm} +\sqrt{
            \frac{1}{2}
            (\sqrt{({b^i}^2-{v^i}^2)^2 + 4} - 
            ({b^i}^2-{v^i}^2))
            })
\end{align*}
It is important to note that during local training, client $c_1$ only updates parameters $a$ and $v$, leaving $b$ unchanged. Therefore, $b^i_{c_1, \infty}=b^i_{c_1, 0}=b^i$. The same principle applies to client $c_2$ with respect to its own parameters.

\subsection{Challenges in Analysis}
\label{subsec:app_thm_challenges}

To the best of our knowledge, we are the first to explore the learning dynamics and convergence behavior of FL algorithms in an over-parameterized model and orthogonal dataset setting. This setup presents challenges not previously addressed in FL analysis. 
(1) \textbf{Non-convexity}: The model is over-parameterized, which leads to a non-convex loss landscape, making optimization challenging without additional constraints or restrictions. To illustrate the non-convex of the loss function, we can evaluate the second derivative of client $c_1$'s loss function $L_1=(av-1)^2$ with respect to the parameters $a$ and $v$. By computing the Hessian matrix and its eigenvalues, we can get insight about the convexity of the function. For instance, at the parameter values $a=0.1$ and $v=0.1$, the Hessian matrix yields one negative eigenvalue, indicating the non-convexity.
(2) \textbf{Hyperbolic local trajectories}: As shown in the previous subsection, the local trajectories of client-side parameters can follow a hyperbolic path. This adds complexity to the analysis of client-side learning dynamics, as the intersection point with the global minimum may not exhibit desirable properties and local convexity may not be clear. 

These challenges make the analysis of \fedavg in an over-parameterized model and orthogonal dataset setting a complex problem. To facilitate theoretical analysis and provide desirable guarantees, our assumptions are necessary. 

\subsection{The Convergence Rate of \fedavg and \fedbug}
\label{subsec:app_thm_convergence}

In this section, we prove that \fedbug converges faster than \fedavg. We begin with the introduction of the three assumptions and then proceed with our two theorems.

\begin{assumption} \label{app_asm_LocalLinear}
\textbf{(Local Linearity)} The server model parameters are initialized in the vicinity of the global minimum, such that the global minimum can be locally approximated as a linear function.
\end{assumption}

\begin{assumption} \label{app_asm_BoundC}
\textbf{(Bounded Local Convexity)} Under Assumption~\ref{app_asm_LocalLinear}, there exist constants $\beta_1 > 0$ and $0 < \beta_2 < 1$, such that for all $n$, the value of ${{v^n}^2}$ is bounded as follows: $\beta_1\beta_2 \leq {{v^n}^2} \leq \beta_1$.
\end{assumption}

\begin{assumption}
\label{app_asm_OrthoTrajectory}
\textbf{(Orthogonal Trajectory and Bounded Error)} Under Assumption~\ref{app_asm_LocalLinear} and Assumption~\ref{app_asm_BoundC}, with a properly chosen step size, the local gradient descent update trajectory can be approximated by a linear trajectory orthogonal to the global minimum. The approximation error, which quantifies the discrepancy between the gradient descent solution and the ideal orthogonal linear path solution, is bounded by $\alpha$ times the length of the orthogonal linear trajectory.
\end{assumption}

Assumption~\ref{app_asm_LocalLinear} ensures that \fedavg and \fedbug starts close enough to the global minimum, enabling convergence without the interference of poor initialization and unbounded convexity. 
Building upon Assumptions~\ref{app_asm_LocalLinear} and~\ref{app_asm_BoundC} further characterizes the local convexity of the objective function around the global minimum. By ensuring that it is bounded, both \fedavg and \fedbug are prevented from diverging or oscillating around the global minimum when using an appropriate step size.
Lastly, Assumption~\ref{app_asm_OrthoTrajectory} guarantees that \fedbug converges along a favorable direction and help us to quantify the improvement brought by single global round, while allowing for quantification of the approximation error. Please refer to the right panel of Figure~\ref{fig:thm_illustration} for illustration. We defer the discussion of limitation to Section~\ref{sec:thm_disscussion}.

In this section, we present our theoretical findings on the convergence behavior of \fedavg. Distinct from conventional FL convergence analysis, we discover that the unique FL setup permits an alternative approach for proving convergence. The cornerstone of our theorem is the observation that at each global round, the discrepancy between clients' parameters diminishes by a factor $r$. Here, $r$ is determined by the degree of alignment between the local update and the axial direction of the last layer parameters – a measure that captures the consistency of local updates with the global model's structure.

To measure this discrepancy reduction ratio, we define two useful terms:

\begin{definition}
Client discrepancy $d^i$ is the L1 distance between the server model parameters $a^i$ and $b^i$ at the $i$-th global round : $d^i = \normx{a^i-b^i}$.
\end{definition}

\begin{definition}
Client discrepancy contraction ratio $r$: $r=d^{i+1}/d^{i}$.
\end{definition}

Now, we present the theorem describing the convergence behavior of \fedavg:
\begin{theorem} \label{app_thm_Fedavg}
Under Assumptions~\ref{app_asm_LocalLinear},~\ref{app_asm_BoundC} and~\ref{app_asm_OrthoTrajectory}, models trained using \fedavg converge with the client discrepancy contraction ratio upper bounded by $\frac{1+\cos^2\theta (1+\alpha)}{2}$, where $\theta$ is the angle between the local update direction and the axis of the last layer parameter.
\end{theorem}
\begin{proof}
The proof aims to show that the client discrepancy contracts after one global round. The contraction ratio is determined by the angle between the axes of last layer parameters and the local update direction (refer to Figure~\ref{fig:thm_illustration} for illustration).

For conciseness, we denote the minima client $\ca$ and $\cb$ reached under Assumption~\ref{app_asm_OrthoTrajectory} as $(\tilde{a}^i_{c_1,*}, b^i, \tilde{v}^i_{c_1,*})$ and $(a^i, \tilde{b}^i_{c_2,*}, \tilde{v}^i_{c_2,*})$, respectively.  We can express the client discrepancy at the $i$-th global round as follows:
\begin{equation}
d^{i+1}
= \normx{a^{i+1}-b^{i+1}}
= \normx{\frac{a^i+a^i_{c_1,*}}{2}-\frac{b^i+b^i_{c_2,*}}{2}}
= \frac{1}{2}d^{i} + \frac{1}{2}\normx{a^i_{c_1,*}-b^i_{c_2,*}}
\end{equation}

Considering the approximation error stated in Assumption~\ref{app_asm_OrthoTrajectory}, we get:
\begin{equation}
\frac{1}{2}\normx{a^i_{c_1,*}-b^i_{c_2,*}}\leq \normx{\tilde{a}^i_{c_1,*}-\tilde{b}^i_{c_2,*}} + e_i, 
\end{equation}
where $0 < e_i \leq \alpha \normx{\tilde{a}^i_{c_1,*}-\tilde{b}^i_{c_2,*}}$ is the approximation error.

To obtain $\normx{\tilde{a}^i_{c_1,*}-\tilde{b}^i_{c_2,*}}$, we notice that it is the projection of the line segment connecting $(\tilde{a}^i_{c_1,*}, \tilde{v}^i_{c_1,*})$ and $(\tilde{b}^i_{c_2,*}, \tilde{v}^i_{c_2,*})$ onto the plane orthogonal to the axial direction of $v$. The line segment, on the other hands, represent the projection of the original client discrepancy on to the linearized global minima. Therefore, $\normx{\tilde{a}^i_{c_1,*}-\tilde{b}^i_{c_2,*}} = d^{i} \cos^2\theta$, as it is a double projection of the original client discrepancy.

Substituting this into the expression for $d^{i+1}$, we get:
\begin{equation}
d^{i+1}
= \frac{1}{2}d^{i} + \frac{1}{2}\normx{a^i_{c_1,*}-b^i_{c_2,*}}
\leq \frac{1}{2}d^{i} + \frac{1}{2}d_{i}\cos^2\theta (1+\alpha)
\end{equation}
This yields the ratio
\begin{equation}
r = \frac{d^{i+1}}{d^{i}} \leq \frac{1+\cos^2\theta(1+\alpha)}{2}
\end{equation}
\end{proof}

We proceed to theoretically demonstrate the superiority of \fedbug over \fedavg by proving that \fedbug exhibits a smaller contraction ratio than \fedavg. To better comprehend the learning behavior, we focus on the case that the last layer parameter $c$ is frozen only for \textit{one local training iteration} and simultaneously perform gradient descent only on parameters $a$ and $b$. Afterwards, all clients unfreeze parameter $v$ and conduct gradient descent on all parameters. Importantly, unlike the learning dynamics of \fedavg, we consider a finite step size in this analysis.

\begin{theorem} \label{app_thm_Fedbug}
Under Assumptions~\ref{app_asm_LocalLinear}, ~\ref{app_asm_BoundC}, and ~\ref{app_asm_OrthoTrajectory}, there exists $1-\beta_2<m<1$ such that if the step size satisfies $\frac{1-m}{\beta_1 \beta_2}<\eta<\frac{1}{\beta_1}$, \fedbug converges with a client discrepancy contraction ratio upper bounded by $\frac{1+m\cos^2\theta(1+\alpha)}{2}$, where $\theta$ is defined as in Theorem~\ref{app_thm_Fedavg}.
\end{theorem}

\begin{proof}
To prove Theorem~\ref{app_thm_Fedbug}, we divide the learning process into two stages. In the first stage, we update the parameters $a^i_{c1,0}$ and $b^i_{c2,0}$ for one step while freezing the last layer parameter $c$. Therefore we have:
\begin{align}
a^i_{c1,1} & = a^i - \eta v^i(a^iv^i-1), \nonumber \\
b^i_{c2,1} & = b^i - \eta v^i(b^iv^i-1). \nonumber
\end{align}

We now compute the distance between the updated parameters:
\begin{equation} \label{eq1}
\normx{a^i_{c1,1} - b^i_{c2,1}} = \normx{(a^i - b^i)(1 - \eta {v^i}^2)}.
\end{equation}

In the second stage, we unfreeze parameter $v$ and perform gradient descent on all the parameters. Then, we can obtain $\normx{\tilde{a}^i_{c_1,*}-\tilde{b}^i_{c_2,*}}=\normx{a^i_{c_1,1}-b^i_{c_2,1}}\cos^2\theta$ similar to Theorem~\ref{app_thm_Fedavg} and get:
\begin{equation}
\begin{split}
d^{i+1}
& \leq \frac{1}{2}d^{i} + \frac{1}{2}\normx{\tilde{a}^i_{c_1,*}- \tilde{b}^i_{c_2,*}}(1+ \alpha) \\
& = \frac{1}{2}d^{i} + \frac{1}{2}\normx{a^i_{c_1,1}-b^i_{c_2,1}}\cos^2\theta(1+\alpha)\\
& = \frac{1}{2}d^{i} + \frac{1}{2}(1-\eta c^{i2})\normx{a^i_{c_1,0}-b^i_{c_2,0}}\cos^2\theta(1+\alpha) \\
\end{split}
\end{equation}
Using Assumption~\ref{app_asm_BoundC} and the condition on step size, we obtain the contraction ratio:
\begin{equation}
r = \frac{1+(1-\eta c^{i2})\cos^2\theta(1+\alpha)}{2}
< \frac{1+m\cos^2\theta(1+\alpha)}{2}
\end{equation}
\end{proof}

\textbf{Implication From the Proof.} Our proof introduces a novel approach to analyzing the convergence behavior of the \fedavg algorithm in an over-parameterized context, contributing a significant insight to the FL community. We extend beyond the traditional single-variable loss function analysis used in previous studies, tackling a more complex multi-variable loss scenario. At the same time, our unique orthogonal task setup also practically captures the client drift caused by dataset heterogeneity. Moreover, we propose a novel measure of client alignment - the contraction ratio $r$, which quantifies the alignment between local updates and the global model's structure, thereby offering a more intuitive understanding of convergence dynamics. Our findings point to a promising future research direction - optimizing local-global alignment to enhance convergence in over-parameterized models and orthogonal regression datasets. Our work not only deepens the understanding of FL algorithms but also paves the way for creating more efficient FL systems.

\subsection{Empirical Validation of Convergence Rate}
\label{subsec:app_thm_exp}

\textbf{Experimental Setup.} We conduct simulation experiments to empirically verify our theorem and the underlying assumptions. In our experiments, we consider a two-client federated learning setup with a two-layer network, following the orthogonal dataset setup described in the theoretical analysis. The model parameters $(a^0,b^0,v^0)$ are initialized uniformly from the interval $[0,2]$, and we repeat the experiments 50 times with different seeds. For experimental details, we perform $80$ global rounds in each experiment. On the local side, we utilize the Stochastic Gradient Descent (SGD) optimizer with a learning rate of 0.1, and each client model is trained for 50 local iterations. To compute the theoretical contraction ratio near the minima, we explicitly estimate $\theta$ using the gradient of the client model and then compare it with the empirical one $r=\frac{d^{i+1}}{d^{i}}$, where $d^{i}$ represents the difference between $a^i$ and $b^i$.

\subsection{Discussions}
\label{sec:thm_disscussion}

\textbf{Relation to FedBABU.} Our analysis extends to FedBABU, in which the last layer parameter $v$ remains constant throughout the training. In this context, the client discrepancy contraction ratio reduces to $\frac{1}{2}$.
Despite the rapid convergence of FedBABU, it encounters a significant challenge. Similar to Proposition $2$ presented in~\cite{lee2023surgical}, we can consider a scenario where a non-linear function (e.g., Rectified Linear Unit) is positioned between the first and second layers and $v$ is initialized with a negative value. In such a situation, FedBABU falls short of achieving the optimum because the model output always remains non-positive, while both \fedavg and \fedbug can still converge to some stationary points. 

\textbf{Discussion About Assumptions.} Our analysis is based on three key assumptions that help us derive tractable convergence results, thereby simplifying the complex task of dealing with the global parameter sequence.

\begin{itemize}
    \item The first assumption (Local Linearity) considers the global minimum as a linear function, allowing for the direct quantification of the client discrepancy and the double projection trick, making the derivation more straightforward. However, we recognize that this approximation introduces a margin of error that could influence the overall model performance and the applicability of the theoretical results. Future work may aim to quantify this error consider first-order or second-order Taylor expansion and determine an upper bound to better understand its impact.
    \item Assumption 2 (Bounded Local Convexity) aligns with previous work in the Federated Learning (FL) domain. It bounds the local convexity, preventing the algorithms (\fedavg and \fedbug) from diverging or oscillating when using an appropriate step size. 
    \item Assumption 3 (Orthogonal Trajectory) builds upon Assumption 1 and explicitly bounds the approximation error to account for potential curvature in the gradient flow. This error is most pronounced at the beginning of the parameter update and could influence the algorithms' convergence trajectory. Interestingly, as the parameters get closer to the linearized global minima, their direction becomes increasingly orthogonal to the global minima, a consequence of the first assumption. Future research should explore how this error evolves during training and how it impacts the final model performance.
\end{itemize}

\textbf{Future Directions.} We briefly explore possible generalization of our presented work as the future direction from the task and model perspective. \textbf{Task.} First of all, we provide a direct generalization of the orthogonal task setup, where we use $\mathcal{T}_1=\{x_1=e_1, y_1=1\}$ and $\mathcal{T}_2=\{x_2=e_2, y_2=1\}$ as datasets. Here,  $e_1$ and $e_2$ are orthogonal vectors, and we can define the model function as $f(x) = x[ae_1,be_2]^\top v$. Besides, we can also consider the case of $m$ clients and redefine the client discrepancy to accommodate such scenarios. Moreover, the effect of dataset orthogonality and non-orthogonality can be studied to understand the effect of the general convergence of FL algorithms. \textbf{Model.} Generalizing the model presents two interesting avenues for research. Firstly, we can consider feature dimensions greater than one, which adds complexity to the analysis due to increased over-parameterization. Secondly, extending the analysis to a multi-layer neural network opens up possibilities for observing phenomena such as acceleration, as previously explored in studies like~\cite{arora2018optimization, shi2022towards}. Investigating these directions will contribute to a more comprehensive understanding of the convergence behavior of FL algorithms.

\newpage
\section{Comparison Between \fedavg and \fedbug}
\label{sec:app_alg}

We provide self-contained outlines of the \fedavg and \fedbug algorithms in Algorithm~\ref{alg:app_fedavg} and Algorithm~\ref{alg:app_fedbug}, respectively. The key difference, highlighted in red and blue, is that while \fedavg updates all $M$ modules at each local iteration, \fedbug unfreezes and updates one module at the beginning, progressively training an additional module every $\frac{PK}{M}$ local iterations. As \fedbug does not require extra information like gradients, momentum, or regularization, it can be easily incorporated into other FL algorithms.

\setlength{\fboxsep}{1pt}
\begin{multicols}{2}
\begin{algorithm}[H]
    \caption{\colorbox{red!20}{\texttt{FedAvg}}}
    \begin{algorithmic}[1]
        \Statex \textbf{Notation}: 
        \Statex \hspace{15px} $\theta^{1:m}$: the first $m$ modules of model $\theta$ 
        \Statex \hspace{15px} $R$: number of global rounds
        \Statex \hspace{15px} $K$: number of local iterations
        \Statex 
        \State \textbf{Input}: global model $\theta$ with $M$ modules
        \For {$r = 1, \dots, R$}
        \State Sample clients $S \subseteq \{1,...,N\}$
        \For {each client $i \in S$ in parallel}
        \State Initialize local model $\theta_{i} \gets \theta$
        \For {$k=1,\dots, K$}
        \State $\theta_{i}^{1:\colorbox{red!20}{$M$}}  \gets \theta_{i}^{1:\colorbox{red!20}{$M$}} - \eta_l\nabla F_i (\theta_{i}^{1:\colorbox{red!20}{$M$}})$
        \EndFor
        \State $\Delta_i \gets \theta_{i} - \theta$
        \EndFor
        \State $\theta \gets \theta + \frac{\eta_g }{\lvert S \rvert}\sum_{i \in \mathcal{S}}\Delta_i$
        \EndFor
    \end{algorithmic}
    \label{alg:app_fedavg}
\end{algorithm}

\columnbreak

\setlength{\columnsep}{+1cm}

\begin{algorithm}[H]
    \caption{\colorbox{blue!20}{\texttt{FedBug}}}
    \begin{algorithmic}[1]
        \Statex \textbf{Notation}: 
        \Statex \hspace{15px} $\theta^{1:m}$: the first $m$ modules of model $\theta$ 
        \Statex \hspace{15px} $R$: number of global rounds
        \Statex \hspace{15px} $K$: number of local iterations
        \Statex \hspace{15px} $P$: gradual unfreezing stage percentage
        \State \textbf{Input}: global model $\theta$ with $M$ modules
        \For {$r = 1, \dots, R$}
        \State Sample clients $S \subseteq \{1,...,N\}$
        \For {each client $i \in S$ in parallel}
        \State Initialize local model $\theta_{i} \gets \theta$
        \For {$k=1,\dots, K$}
        \State \colorbox{blue!20}{$m \gets \min \{M, \lceil \frac{kM}{PK} \rceil \}$}
        \State $\theta_{i}^{1:\colorbox{blue!20}{$m$}}  \gets \theta_{i}^{1:\colorbox{blue!20}{$m$}} - \eta_l\nabla F_i (\theta_{i}^{1:\colorbox{blue!20}{$m$}})$ 
        \EndFor
        \State $\Delta_i \gets \theta_{i} - \theta$
        \EndFor
        \State $\theta \gets \theta + \frac{\eta_g }{\lvert S \rvert}\sum_{i \in \mathcal{S}}\Delta_i$
        \EndFor
    \end{algorithmic}
    \label{alg:app_fedbug}
\end{algorithm}
\end{multicols}

\medskip

\end{document}